%% file: main_arxiv.tex
\documentclass[final]{cvpr}

\usepackage{times}
\usepackage{epsfig}
\usepackage{graphicx}
\usepackage{amsthm}
\usepackage{amsmath}
\usepackage{amssymb}

\usepackage{bm}
\usepackage{booktabs}
\usepackage[title]{appendix}

\usepackage[pagebackref=true,breaklinks=true,colorlinks,bookmarks=false]{hyperref}

\def\cvprPaperID{2408} 
\def\confYear{CVPR 2021}

\newtheorem{proposition}{Proposition}

\newcommand{\p}[1]{\textbf{#1.}}

\begin{document}

\title{Learning View-Disentangled Human Pose Representation by Contrastive Cross-View Mutual Information Maximization}

\author{Long Zhao$^{1,}$\thanks{This work was done while the author was a research intern at Google.} \quad Yuxiao Wang$^2$ \quad Jiaping Zhao$^2$ \quad Liangzhe Yuan$^2$ \quad Jennifer J.~Sun$^3$\\
Florian Schroff$^2$ \quad Hartwig Adam$^2$ \quad Xi Peng$^4$ \quad Dimitris Metaxas$^1$ \quad Ting Liu$^2$\\
\\
$^1$Rutgers University \quad $^2$Google Research \quad $^3$Caltech \quad $^4$University of Delaware\\
}

\maketitle

\input{0_abstract}

\input{1_introduction}
\input{2_related_work}
\input{3_approach}
\input{4_experiments}
\input{5_conclusion}

\begin{appendices}

\input{6_supp_theoretical_results}
\input{7_supp_implementation_details}
\input{8_supp_additional_results}

\end{appendices}

{\small
\bibliographystyle{ieee_fullname}
\bibliography{egbib}
}

\end{document}

%% file: 0_abstract.tex
\begin{abstract}
   We introduce a novel representation learning method to disentangle pose-dependent as well as view-dependent factors from 2D human poses. The method trains a network using cross-view mutual information maximization (CV-MIM) which maximizes mutual information of the same pose performed from different viewpoints in a contrastive learning manner. We further propose two regularization terms to ensure disentanglement and smoothness of the learned representations. The resulting pose representations can be used for cross-view action recognition.
   
   To evaluate the power of the learned representations, in addition to the conventional fully-supervised action recognition settings, we introduce a novel task called single-shot cross-view action recognition. This task trains models with actions from only one single viewpoint while models are evaluated on poses captured from all possible viewpoints. We evaluate the learned representations on standard benchmarks for action recognition, and show that (i) CV-MIM performs competitively compared with the state-of-the-art models in the fully-supervised scenarios; (ii) CV-MIM outperforms other competing methods by a large margin in the single-shot cross-view setting; (iii) and the learned representations can significantly boost the performance when reducing the amount of supervised training data. Our code is made publicly available at \url{https://github.com/google-research/google-research/tree/master/poem}.
\end{abstract}

%% file: 1_introduction.tex
\section{Introduction}

\begin{figure}[t]
\begin{center}
  \includegraphics[width=\linewidth]{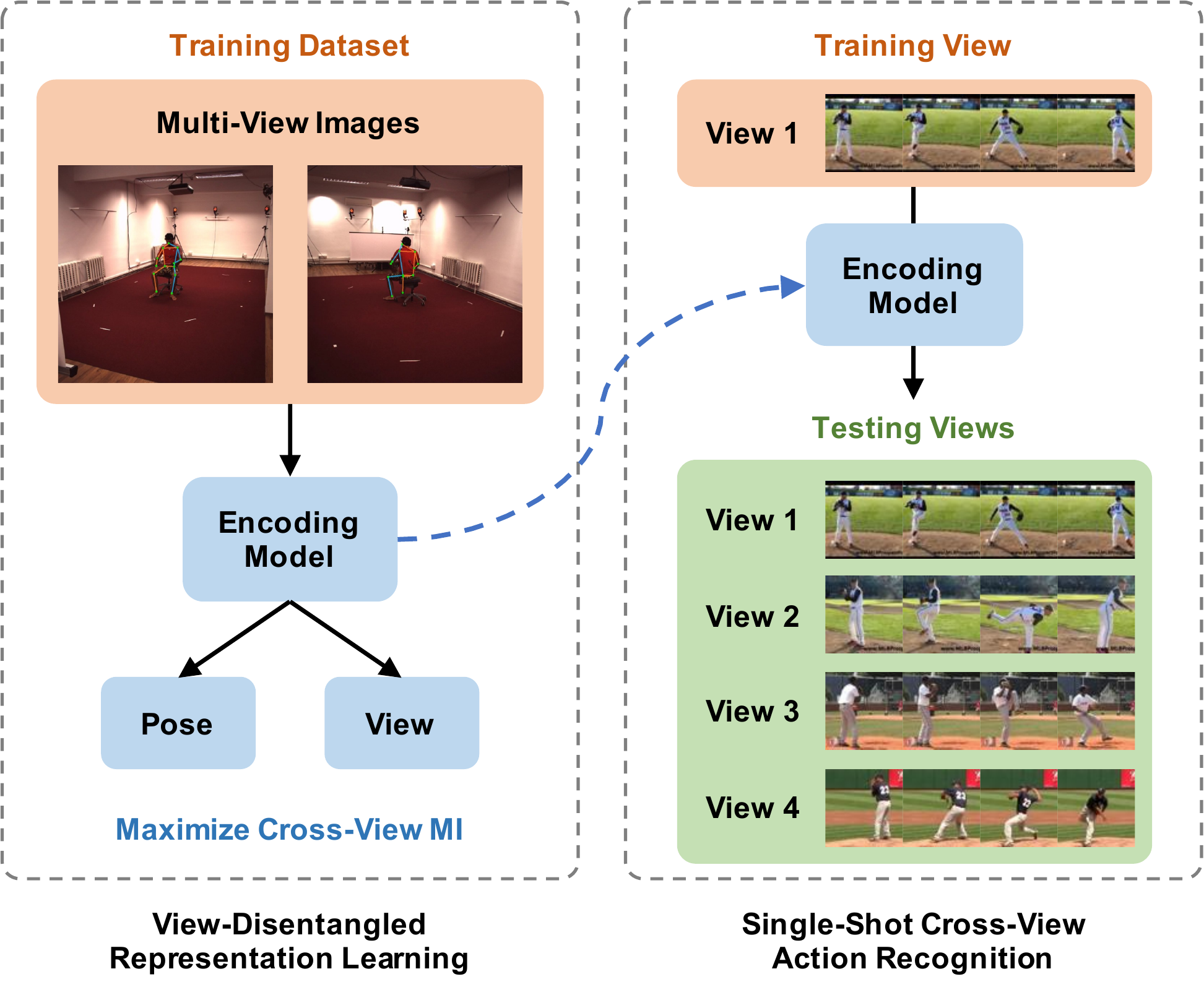}
\end{center}
  \vspace{-2.6mm}
\caption{{\bf Left:} We propose to learn view-disentangled representation for human poses by maximizing cross-view mutual information. {\bf Right:} The learned representation can be applied to downstream tasks such as single-shot cross-view action recognition.}
\label{fig:xview}
\end{figure}

Understanding human poses and actions is a fundamental problem in computer vision due to its broad applications in the real world, such as video content analysis, intelligent photography, AR/VR techniques, and human-computer interface. Recently, remarkable improvements have been achieved with deep learning approaches~\cite{du2017rpan,li2018co,nie2015joint,zhao2018learning,zhao2019semantic}. However, these data-driven approaches are usually vulnerable to changes of viewpoints. 
In particular, testing-time unseen viewpoints often lead to significant degradation in recognition performance~\cite{wang2014cross}.

To mitigate this issue, methods for cross-view action recognition~\cite{wang2014cross} have been proposed, where models are trained with a set of actions captured from different viewpoints simultaneously so that they can be applied to novel views unseen from training at testing time. Previous studies usually require extensive supervision from multiple views to learn view-invariant features~\cite{wang2018dividing,wang2014cross} or transferable representations~\cite{li2019episodic,liu2011cross} for action recognition. 
Collecting labeled action data at scale from multiple views can be expensive and challenging in the wild due to potential limitation of camera placement, scene and actor setup, \etc.

We address this challenge by proposing a novel view-disentangled representation learning approach. To train the representation model, we only require pairs of 2D poses captured from different viewpoints without additional task-relevant supervision, which are widely available in standard multi-view human action datasets~\cite{ionescu2013human3,shahroudy2016ntu}. Our target is to disentangle pose-dependent (view-invariant) and view-dependent representations from 2D poses, which has not been well-explored in existing works.

To achieve this, we train a representation-learning function, \ie, an encoder, following the Mutual Information (MI) maximization principle~\cite{bell1995information}. Specifically, in order to fulfill the view-disentanglement constraint, we propose to maximize the cross-view MI, \ie, the dependency between learned representations of the same pose from different views. In addition, we theoretically motivate two regularization terms that encourage disentanglement and smoothness of the learned representation to further improve its power. Our objective is optimized in a contrastive manner based on recent advances made in MI estimation~\cite{belghazi2018mutual,cheng2020club,hjelm2019learning,oord2018representation,ozair2019wasserstein}. Compared to approaches based on cross reconstruction~\cite{nie2020unsupervised}, the proposed approach yields stronger representative powers by using negative training pairs which provide an additional source of supervision~\cite{chen2020simple,hjelm2019learning}.

We show that the resulting pose representation can be used for action recognition in a fully-supervised setting. To further demonstrate its view-disentangled property, we introduce a novel and more challenging task, namely, single-shot cross-view action recognition. In this setting, recognition models are trained with 2D poses from one single view but expected to generalize to unseen views at testing time. This setting is highly practical for real-world applications: it only requires collecting training data from a fixed camera view, and the resulting recognition model can be applied to various difference views. Note that success in this task requires not only discriminative but also view-invariant representations for 2D poses. Fig.~\ref{fig:xview} summarizes our representation learning framework and its application to single-shot cross-view action recognition downstream task.

To sum up, our main contributions of this work include: (i) a novel objective to learn view-disentangled representation for 2D human poses by maximizing cross-view MI; (ii) two regularization techniques to guarantee disentanglement and smoothness of the learned representation; (iii) a newly proposed task called single-shot cross-view action recognition that can be used for evaluating view-invariant representation for human poses. We evaluate the proposed method on standard benchmarks for action recognition under scenarios of full-supervision, single-shot cross-view setting, and supervision with limited data. Experimental results show that our approach is comparable to the state of the art in the fully-supervised setting, while it can consistently and statistically significantly outperform competing methods under the other two scenarios.

%% file: 2_related_work.tex
\section{Related Work}

\p{Representation Learning} 
There has been a lot of recent progress on learning representations for visual and temporal data~\cite{bachman2019learning,chen2020simple,he2020momentum,han2019video,hjelm2019learning,oord2018representation}. These representations are often trained using unsupervised or self-supervised approaches. In particular, approaches based on contrastive learning (contrasting positive pairs with negative pairs) has been shown to be effective at learning visual representations~\cite{chen2020simple,han2019video,oord2018representation,tian2020contrastive}. Recent works have further investigated contrastive training objectives based on MI maximization, such as maximizing MI between different augmented ``views" of the same image~\cite{bachman2019learning} and between local and global features~\cite{hjelm2019learning,oord2018representation}. 

Our work aims to learn a representation on 2D poses instead of images and we apply contrastive learning to maximize MI of pose representations across camera views. Previous works on representation learning for 2D poses~\cite{sun2019view} have focused on studying the view-invariance property with triplet loss~\cite{schroff2015facenet}. In contrast, we differ in our goal, \ie, disentangling pose-dependent and view-dependent factors, and our approach, \ie, contrastive loss with MI maximization across camera views.





\p{View Disentanglement} When depicted in 2D space, human poses can differ in appearance due to changes in pose and changes in viewpoints. The ability to disentangle pose-dependent and view-dependent factors from human poses as well as objects are useful for a variety of downstream tasks, including video alignment~\cite{dwibedi2019temporal}, human re-identification~\cite{zheng2019pose}, action recognition~\cite{nie2020unsupervised}, and object classification \cite{ho2019pies}. Here, we explore view-disentanglement for pose-based action recognition.

Some studies focus on learning view-invariant representations for human poses~\cite{sun2019view} and objects~\cite{ho2019pies}. The learned representation space in these works is pose-dependent, and not view-dependent. Other works~\cite{lorenz2019unsupervised,nie2020unsupervised,peng2017reconstruction,rhodin2019neural,rhodin2018unsupervised,skafte2019explicit} also learns disentangled representations for human poses and images. In particular, \cite{rhodin2019neural,rhodin2018unsupervised} learns a 3D geometry-aware representation space for pose images, where rotation matrices can be applied to the representation to generate images from new views; \cite{lorenz2019unsupervised,skafte2019explicit} disentangle shape and appearance using generative modeling on images, and does not explicitly disentangle viewpoints. Closest to our work, \cite{nie2020unsupervised} also learns to disentangle poses and views on human poses. Our work differs in that we embed 2D poses, which can be extracted from images without using camera parameters, while their method relies on ground truth 3D poses as input. Additionally, instead of cross reconstruction, our approach is based on contrastive loss and MI maximization.


\begin{figure*}[t]
\begin{center}
  \includegraphics[width=.98\linewidth]{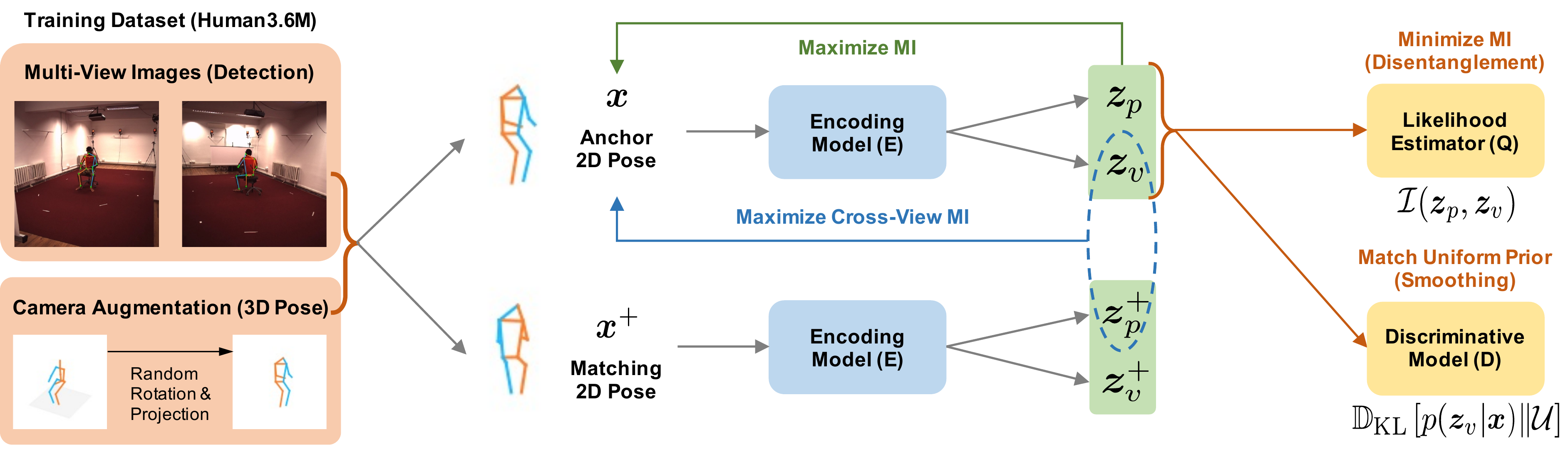}
\end{center}
  \vspace{-.8em}
  \caption{Overview of CV-MIM model training pipeline. Our model takes pairs of multi-view 2D poses detected from images or camera augmentation (optional) and produces pose representations $\bm{z}_p$ and view representations $\bm{z}_v$. $E$, $Q$ and $D$ are optimized alternatively.}
\label{fig:overview}
\end{figure*}

\p{Cross-View Learning for Poses} There have been great progress in cross-view action recognition with RGB videos. In the multi-view environment, many research works aim to address the issue of view invariance motivated by the availability of different modalities such as pose and depth. In this stream of research, most works use multiple modalities, \eg, RGB~\cite{li2018unsupervised,vyas2020multi,zhang2018adding}, depth~\cite{li2018unsupervised,rahmani2016histogram,vyas2020multi}, RGB+D~\cite{shahroudy2017deep}, or skeleton data~\cite{wang2018dividing,yan2018spatial,zhang2019view}. Another stream of research works~\cite{wang2014cross,li2019episodic} are interested in learning models by using poses from different camera views for cross-view action recognition. Our work belongs to this stream.

Multi-view pose information has been used for learning image representations for 3D pose estimation~\cite{rhodin2018unsupervised,rhodin2019neural} as well. In these works, the representation is trained to reconstruct images of poses from different camera views. While we also leverage multi-view pose data, our method is based on cross-view MI maximization. It can utilize negative samples during optimization for representation learning. 

%% file: 3_approach.tex
\section{Approach}

We begin by summarizing the concept of mutual information and, along the way, introduce the notations. Mutual Information~(MI) is a fundamental measurement to quantify the relationship between random variables. Formally, it measures the dependence of two random variables $\bm{x}$ and $\bm{y}$:
\begin{equation}
\label{eq:mi}
\mathcal{I}(\bm{x}; \bm{y}) = \mathbb{E}_{p(\bm{x}, \bm{y})} \left[ \log \frac{p(\bm{x}, \bm{y})}{p(\bm{x})p(\bm{y})} \right],
\end{equation}
where $p(\bm{x}, \bm{y})$ is the joint probability distribution, while $p(\bm{x})$ and $p(\bm{y})$ are their marginals. In the context of self-supervised representation learning, mutual information can act as a measure of true dependence between observed data samples and learned representations. The objective is to maximize Eq.~\eqref{eq:mi} so that the learned representations retain the most information about the underlying data~\cite{bell1995information,hjelm2019learning,ozair2019wasserstein}. 

This work extends MI maximization principle to view-disentangled representation learning for human poses where view-dependent and pose-dependent representations are learned concurrently for an input 2D pose. As we will show, this objective can be obtained by maximizing cross-view MI. An overview of our approach is presented in Fig.~\ref{fig:overview}.

\subsection{Cross-View Mutual Information Maximization}

To set the stage, we let $\bm{x}^i \in \mathbb{R}^{2 \times N}$ denote the given 2D pose from the $i$-th view, where $N$ is the number of keypoints for representing the pose. We are interested in learning an encoding network $E$ producing a view representation $\bm{z}_v^i \in \mathbb{R}^d$ and a pose representation $\bm{z}_p^i \in \mathbb{R}^d$ from the input $\bm{x}^i$, while $\bm{z}_p^i$ and $\bm{z}_v^i$ are expected to be disentangled (mutually excluded). We define $(\bm{z}_p^j, \bm{z}_v^i)$ as the cross-view representation of a given 2D pose $\bm{x}^i$ from the $j$-th view, and it captures the amount of pose information that can be maintained from a different viewpoint. Then we can obtain our objective for view-disentangled representation learning:
\begin{equation}
\label{eq:cvmi}
\max \left[\sum_{i} \underbrace{\mathcal{I}(\bm{x}^i; \bm{z}_p^i, \bm{z}_v^i)}_\text{MI} + \sum_{i \neq j} \underbrace{\mathcal{I}(\bm{x}^i; \bm{z}^j_p, \bm{z}^i_v)}_\text{Cross-View MI}\right],
\end{equation}
where the first term is the conventional MI-based representation objective, and the second term which is proposed in this work maximizes the MI between the input 2D pose and its cross-view representations. Next, we show how this objective can be further simplified.

As the optimal view and pose representations are disentangled, we assume that $\bm{z}_p$ and $\bm{z}_v$ are simultaneously independent and conditionally independent, \ie, $p(\bm{z}_p, \bm{z}_v) = p(\bm{z}_p)p(\bm{z}_v)$ and $p(\bm{z}_p, \bm{z}_v|\bm{x}) = p(\bm{z}_p|\bm{x})p(\bm{z}_v|\bm{x})$. Based on this assumption, it can be easily shown that,
\begin{equation}
\label{eq:dis}
\mathcal{I}(\bm{x}^i; \bm{z}_p^j, \bm{z}_v^i) = \mathcal{I}(\bm{x}^i; \bm{z}_p^j) + \mathcal{I}(\bm{x}^i; \bm{z}_v^i).
\end{equation}
Then, by the Data Processing Inequality~\cite{cover2012elements}, we have that,
\begin{equation}
\label{eq:bound}
\begin{aligned}
\mathcal{I}(\bm{x}^i; \bm{z}_p^j, \bm{z}_v^i) &\geqslant \mathcal{I}(\bm{z}^i_p, \bm{z}^i_v; \bm{z}_p^j) + \mathcal{I}(\bm{z}^i_p, \bm{z}^i_v; \bm{z}_v^i)\\
&= \mathcal{I}(\bm{z}^i_p; \bm{z}^j_p) + \mathcal{I}(\bm{z}_v^i; \bm{z}_v^i)\\
&= \mathcal{I}(\bm{z}^i_p; \bm{z}^j_p) + \mathcal{H}(\bm{z}_v^i)\\
&\geqslant \mathcal{I}(\bm{z}^i_p; \bm{z}^j_p), 
\end{aligned}
\end{equation}
where $\mathcal{H}$ is the Shannon entropy. The equality in the second line is achieved since $\bm{z}_p$ and $\bm{z}_v$ are independent, and the last inequality holds due to the non-negativity of entropy.

After combing Eq.~\eqref{eq:bound} with Eq.~\eqref{eq:cvmi}, we achieve a relaxed formulation of cross-view MI maximization:
\begin{equation}
\label{eq:simple_cvmi}
\max \left[\sum_{i} \mathcal{I}(\bm{x}^i; \bm{z}_p^i, \bm{z}_v^i) + \sum_{i \neq j} \mathcal{I}(\bm{z}^i_p; \bm{z}^j_p)\right].
\end{equation}
The above objective is a lower bound of Eq.~\eqref{eq:cvmi}. Intuitively, the second term aims to maximize MI of the same pose but performed from different viewpoints. 

Both terms in Eq.~\eqref{eq:simple_cvmi} can be optimized by MI estimators~\cite{belghazi2018mutual,hjelm2019learning,oord2018representation,ozair2019wasserstein} which estimate a lower-bound of MI by training a classifier in a contrastive learning objective, \ie, it distinguishes between samples coming from the joint distribution $p(\bm{x}, \bm{z})$ and the product of marginals $p(\bm{x})p(\bm{z})$ of the input pose $\bm{x}$ and target representation $\bm{z}$ encoded by $E$. We use the Jensen-Shannon MI estimator~\cite{hjelm2019learning} to maximize Eq.~\eqref{eq:simple_cvmi} since it achieves a good balance between computational efficiency and performance. Given the representation $\bm{z}'$ which is a negative match of the input $\bm{x}$, maximizing $\mathcal{I}(\bm{x}; \bm{z})$ is equivalent to minimizing the following loss:
\begin{multline}
\label{eq:est_mi}
\min_E \mathcal{L}_\text{MI}(\bm{x}; \bm{z}) = \mathbb{E}_{(\bm{x}, \bm{z}) \sim p(\bm{x}, \bm{z})}[\rho(-f(\bm{x}, \bm{z}))] \\ - \mathbb{E}_{(\bm{x}, \bm{z}') \sim p(\bm{x})p(\bm{z})}[-\rho(f(\bm{x}, \bm{z}'))],
\end{multline}
where $\rho(x) = \log (1 + e^x)$ denotes the softplus activation, and $f$ is a discriminator function modeled by a network.

\subsection{Representation Disentanglement}

The major assumption we made when deriving Eq.~\eqref{eq:dis} is that $\bm{z}_v$ (view representation) and $\bm{z}_p$ (pose representation) are disentangled during optimization, \ie, they are simultaneously and conditionally independent. Therefore, we introduce a regularization term $\mathcal{L}_\text{inter}$ to guarantee disentanglement between $\bm{z}_v$ and $\bm{z}_p$ based on their MI. By minimizing $\mathcal{L}_\text{inter}$, we encourage the information in these two random variables are mutually exclusive. 

However, lower-bound MI estimators are inapplicable to disentanglement because  they are inconsistent to MI minimization tasks. Hence, we leverage the contrastive log-ratio upper-bound MI estimator~\cite{cheng2020club} which estimates the probability log-ratio between the conditional log-likelihood of positive sample pair $\log p(\bm{z}_v|\bm{z}_p)$ and negative sample pair $\log p(\bm{z}_v'|\bm{z}_p)$. Unfortunately, the conditional 
relation between $\bm{z}_v$ and $\bm{z}_p$ is unavailable in our case. To address this issue, we use a variational distribution $q(\bm{z}_v|\bm{z}_p)$ which is predicted by a neural network $Q$ to approximate $p(\bm{z}_v|\bm{z}_p)$. After combining all these together, we reach the following objective function for the encoder $E$:
\begin{multline}
\min_E \mathcal{L}_\text{inter}(\bm{z}_p; \bm{z}_v) = \mathbb{E}_{(\bm{z}_p, \bm{z}_v) \sim p(\bm{z}_p, \bm{z}_v)}[\log q(\bm{z}_v|\bm{z}_p)] \\ - \mathbb{E}_{(\bm{z}_p, \bm{z}_v') \sim p(\bm{z}_p)p(\bm{z}_v)}[\log q(\bm{z}_v'|\bm{z}_p)].
\end{multline}
While at the same time, $Q$ is trained to minimize the KL-divergence between the true conditional probability distribution $p(\bm{z}_v|\bm{z}_p)$ and variational one $q(\bm{z}_v|\bm{z}_p)$:
\begin{equation}
\label{eq:var}
\min_Q \mathcal{L}_\text{KL}(\bm{z}_p, \bm{z}_v) = \mathbb{D}_\text{KL}\left[q(\bm{z}_v|\bm{z}_p) \middle\| p(\bm{z}_v|\bm{z}_p)\right].
\end{equation}
For simplicity, we assume $q(\bm{z}_v|\bm{z}_p)$ follows a Gaussian distribution in this work, and then Eq.~\eqref{eq:var} can be efficiently solved by maximum likelihood estimation.

\subsection{Representation Smoothing}

From Eq.~\eqref{eq:bound} we can see that maximizing the entropy of the view representation $\mathcal{H}(\bm{z}_v)$ is also desirable during optimization. However, this term is intractable due to the high dimensionality of the representation space. As we are not concerned with its precise value, we present an alternative maximization strategy from the aspect of prior matching.

Given a bounded interval $[a, b]$, entropy is maximised when the probability distribution is uniform. By this observation and $\mathcal{H}(\bm{z}_v) \geqslant \mathcal{H}(\bm{z}_v|\bm{x})$, we impose the maximum-entropy constraint onto learned representations by implicitly training the encoder $E$ so that the push-forward distribution $p(\bm{z}_v|\bm{x})$ matches a uniform prior $\mathcal{U}(a, b)$:
\begin{equation}
\label{eq:prior_div}
\min \left\{\mathbb{D}_\text{KL}\left[p(\bm{z}_v|\bm{x}) \middle\| \mathcal{U}(a, b)\right]\right\}.
\end{equation}
This is achieved by training a discriminator $D$ to estimate the divergence in Eq.~\eqref{eq:prior_div}, and then training the encoder $E$ to minimize this estimation. They play the minimax game:
\begin{multline}
\min_E \max_D \mathcal{L}_\text{prior}(\bm{z}) = \mathbb{E}_{\bm{z} \sim \mathcal{U}(a, b)}[\log D(\bm{z})] \\ + \mathbb{E}_{\bm{z} \sim p(\bm{z} | \bm{x})}[\log(1 - D(\bm{z}))].
\end{multline}
To keep it simple, we optimize this loss term on $[0, 1]$ which is done by setting the prior to $\mathcal{U}(0, 1)$ and re-scaling representations via a sigmoid activation. In practice, we match both view and pose representations to this prior since it also benefits the regularization of pose dimensions.

Intuitively, the loss term $\mathcal{L}_\text{prior}$ ensures the learned representations to be smooth~\cite{wang2020understanding}, as we do not assume any special prior on human poses and camera views. Compared with previous works~\cite{alemi2017deep,higgins2017beta,hjelm2019learning,kingma2014auto,sun2019view} that also target representation regularization, our approach provides a more intuitive motivation in favor of uniform prior over other common priors, \eg, a Gaussian distribution.

\subsection{Full Objective}

All three objectives, \ie, MI maximization, representation disentanglement and smoothing (prior matching), can be used together, and doing so we arrive at our full objective for \emph{Cross-View Mutual Information Maximization}~(CV-MIM). We let $\bm{x}^+$ represent a positive match of the input pose $\bm{x}$ which shares the same action but performed from another view, and $\bm{z}_p^+$ be its pose representation, then the complete objective is defined as:
\begin{multline}
\label{eq:full_obj}
\min_E \left[\mathcal{L}_\text{MI}(\bm{x}; \bm{z}_p \odot \bm{z}_v) + \lambda_1 \mathcal{L}_\text{MI}(\bm{z}_p; \bm{z}_p^+)\right] \\ + \min_E \lambda_2 \mathcal{L}_\text{inter}(\bm{z}_p; \bm{z}_v) + \min_Q \mathcal{L}_\text{KL}(\bm{z}_p, \bm{z}_v) \\ + \min_E \max_D \lambda_3 \mathcal{L}_\text{prior}(\bm{z}_p \oplus \bm{z}_v),
\end{multline}
where $\lambda_1$, $\lambda_2$, and $\lambda_3$ are positive parameters that balance the magnitude of each term; $\odot$ is a pre-defined fusion operation which combines pose and view representations; and $\oplus$ denotes concatenation. Note that $E$, $Q$ and $D$ are optimized in an alternative way during network training.

\p{Discussion} It is worth discussing two important properties of our formulation. First, our approach differs from cross-reconstruction based methods~\cite{nie2020unsupervised,peng2017reconstruction,rhodin2019neural,rhodin2018unsupervised} from two perspectives: (i) we do not explicitly perform reconstruction of the input in the objective, which is proven to be a lower bound of MI~\cite{hjelm2019learning}; (ii) our objective trains the model in a contrastive manner, where negative sample pairs are involved to provide additional supervisions and thus manage to improve the power of representation learning.

Second, in addition to Eq.~\eqref{eq:full_obj}, an alternative way of cross-view MI maximization is to optimize Eq.~\eqref{eq:cvmi} directly through lower-bound MI estimators. However, we find this alternate leads to significant performance degradation in the experiments. This is due to the fact that lower-bound MI estimators are not accurate estimations of MI which suffer from high variance~\cite{song2020understanding}. Instead, our formulation is able to address this drawback by decomposing the single objective into multiple simpler criteria.

%% file: 4_experiments.tex
\section{Experiments}

\subsection{Datasets}

\p{Human3.6M} Ionescu~\etal~\cite{ionescu2013human3} built the in-lab dataset with synchronized multi-view images and 3D poses. 
We follow the standard protocol in the literature and use all four camera views of subjects S1, S5, S6, S7, and S8 for model training. Note that this dataset is only used for learning pose representation in this work, \ie, training the encoder, where we do not use any action labels.

We experiment on the following two datasets for action recognition in the fully-supervised scenario where training sets include all views, and the single-shot cross-view setting where actions from only one view are used for training. 

\p{Penn Action} The Penn Action dataset~\cite{zhang2013actemes} consists of 2,326 video sequences of 15 action categories captured from four different views. 
We follow the official training/testing split~\cite{zhang2013actemes} and~\cite{nie2015joint} to remove the action of playing guitar and several videos due to target person invisibility. All videos are up-sampled to 332 frames for action recognition. In the proposed single-shot cross-view setting, videos from one single view in the training set are leveraged for training and videos from all views are used for testing. The final performance is measured by the average top-1 accuracy over all views. 

\p{NTU-RGB+D} This dataset~\cite{shahroudy2016ntu} contains 56,000 video clips in 60 action classes performed by 40 actors captured in-lab environments. Each clip has at most two subjects. Three cameras are used for recording different horizontal views simultaneously, and each action is performed twice towards the left and right cameras, respectively. Thus, there are six views contained in this dataset. Following~\cite{yan2018spatial}, we pad every clip by replaying the sequence from the start to have 300 frames. Furthermore, we only use single-person action categories and the main actor in each video.

There are two common evaluation benchmarks~\cite{shahroudy2016ntu} for action recognition on this dataset. In cross-subject benchmark, 40 subjects are split into training and testing groups, where each group consists of 20 subjects. In cross-view benchmark, training clips come from the second and third cameras, while the evaluation clips are all from the first camera. In this work, we introduce a new evaluation setting for single-shot cross-view action recognition. Specifically, we divide the training set of cross-subject benchmark into six splits according to cameras and replication numbers so that actions performed from only one view are contained in each split, and testing groups including all views and remaining subjects are used for evaluation. We report the average performance of all models trained using the six splits.

\begin{table*}[t]
\begin{center}
\resizebox{.76\linewidth}{!}{
\begin{tabular}{c|c|cccc|c}
\toprule
Methods & VD & Left & Right & Front & Back & Average\\
\midrule
Res-TCN~\cite{kim2017interpretable} & & 86.83 $\pm$ 0.50 & 90.80 $\pm$ 1.09 & 75.99 $\pm$ 1.22 & 75.23 $\pm$ 3.05 & 82.21 $\pm$ 0.71\\
Temporal ConvNet & & 82.78 $\pm$ 1.38 & 88.78 $\pm$ 1.35 & 72.69 $\pm$ 1.83 & 69.70 $\pm$ 1.52 & 78.49 $\pm$ 0.91\\
\midrule
Auto-Encoder & \checkmark & 85.89 $\pm$ 0.46 & 90.55 $\pm$ 0.85 & 77.45 $\pm$ 2.02 & 87.98 $\pm$ 0.93 & 85.47 $\pm$ 0.75\\
VAE~\cite{kingma2014auto} & \checkmark & 87.22 $\pm$ 1.19 & 92.14 $\pm$ 0.39 & 75.87 $\pm$ 2.14 & 88.76 $\pm$ 1.14 & 86.00 $\pm$ 0.72\\
$\beta$-VAE~\cite{alemi2017deep,higgins2017beta} & \checkmark & 85.86 $\pm$ 1.27 & 90.03 $\pm$ 1.24 & 75.83 $\pm$ 1.06 & 83.56 $\pm$ 1.58 & 83.82 $\pm$ 0.60\\
InfoNCE~\cite{oord2018representation} & \checkmark & 87.47 $\pm$ 0.85 & 89.25 $\pm$ 0.74 & 73.30 $\pm$ 0.59 & 83.05 $\pm$ 1.55 & 83.27 $\pm$ 0.55\\
DIM~\cite{hjelm2019learning} & \checkmark & 81.67 $\pm$ 0.70 & 82.64 $\pm$ 0.67 & 76.08 $\pm$ 1.75 & 80.17 $\pm$ 1.29 & 80.14 $\pm$ 0.67\\
Pr-VIPE~\cite{sun2019view} & & 90.06 $\pm$ 0.38 & 89.36 $\pm$ 0.68 & 85.11 $\pm$ 0.69 & 91.58 $\pm$ 0.76 & 89.03 $\pm$ 0.40\\
\midrule
\textbf{CV-MIM} & \checkmark & \textbf{91.82 $\pm$ 0.30} & \textbf{93.73 $\pm$ 0.27} & \textbf{88.81 $\pm$ 0.53} & \textbf{92.65 $\pm$ 0.32} & \textbf{91.75 $\pm$ 0.24}\\
\bottomrule
\end{tabular}}
\end{center}
\caption{Classification accuracy (\%) and standard deviation of models on Penn Action~\cite{zhang2013actemes} with the setting of single-shot cross-view action recognition. Each time, models are trained using one of the left, right, front, and back views, and evaluated on all four views. We highlight view-disentangled (VD) methods. Results are averaged over five runs; best performances are highlighted in bold.}
\label{tbl:penn_sscv}
\end{table*}

\begin{table}[t]
\begin{center}
\resizebox{.92\linewidth}{!}{
\begin{tabular}{c|c|ccc|c}
\toprule
Methods & VD & RGB & Flow & Pose & Accuracy\\
\midrule
Nie~\etal~\cite{nie2015joint} & & \checkmark & & \checkmark & 85.5\\
Cao~\etal~\cite{cao2017body} & & & \checkmark & \checkmark & 95.3\\
& & \checkmark & \checkmark & & 98.1\\
Du~\etal~\cite{du2017rpan} & & \checkmark & \checkmark & \checkmark & 97.4\\
Liu~\etal~\cite{liu2018recognizing} & & \checkmark & & \checkmark & 91.4\\
Luvizon~\etal~\cite{luvizon2020multi} & & \checkmark & & \checkmark & 98.7\\
Res-TCN~\cite{kim2017interpretable} & & & & \checkmark & 98.8\\
Temporal ConvNet & & & & \checkmark & 98.5\\
\midrule
Auto-Encoder & \checkmark & & & \checkmark & 97.7\\
VAE~\cite{kingma2014auto} & \checkmark & & & \checkmark & 97.6\\
$\beta$-VAE~\cite{alemi2017deep,higgins2017beta} & \checkmark & & & \checkmark & 97.7\\
InfoNCE~\cite{oord2018representation} & \checkmark & & & \checkmark & 97.5\\
DIM~\cite{hjelm2019learning} & \checkmark & & & \checkmark & 97.3\\
Pr-VIPE~\cite{sun2019view} & & & & \checkmark & \textbf{98.4}\\
\midrule
\textbf{CV-MIM} & \checkmark & & & \checkmark & \underline{98.1}\\
\bottomrule
\end{tabular}}
\end{center}
\caption{Comparisons of top-1 action recognition accuracy (\%) on Penn Action~\cite{zhang2013actemes}. The check marks indicate the input to each method, including image pixels (RGB), optical flow (Flow), and model-estimated 2D pose (Pose). Top two performances of representation learning models are highlighted in bold and underline.}
\label{tbl:penn_fs}
\end{table}

\subsection{Implementation Details}

Our approach does not require a particular 2D pose estimator, as long as it is reasonable accurate. We use~\cite{papandreou2017towards} in our experiments. All detected keypoints of a 2D pose are then converted into a skeleton representation that consists of 13 joints according to the keypoint definition in~\cite{sun2019view}. We treat two poses as a positive pair if they are projected from the same 3D pose.


\p{Camera Augmentation} We perform camera augmentation to improve the model robustness to large variations in camera viewpoints when applied to downstream tasks. When we train only with detected 2D keypoints in training images, we are constrained to the camera views in the training set. To reduce overfitting to these camera views, we perform camera augmentation by generating projected 2D keypoints from 3D poses at random views. For random rotation in camera augmentation, we follow~\cite{sun2019view} and uniformly sample azimuth angle between $\pm 180^\circ$, elevation between $\pm 30^\circ$, and roll between $\pm 30^\circ$. During model training, we use an even mixture of detected and projected 2D keypoints from different views to form positive 2D pose pairs.

\p{Network Training} The backbone network architecture for our model is based on~\cite{martinez2017simple}. We use $d = 32$ for both pose and view representations as a good trade-off between model size and accuracy. The discriminator function $f$ in Eq.~\eqref{eq:est_mi} is implemented by the encode-and-dot-product architecture~\cite{hjelm2019learning} which enables us to use large numbers of positive/negative samples, and mixture-of-experts~\cite{shi2019variational} is employed for representation fusion. To weigh different losses in Eq.~\eqref{eq:full_obj}, we set $\lambda_1 = 5.0$, $\lambda_2 = 0.5$, and $\lambda_3 = 1.0$ such that all the loss terms have the same order of magnitude and did not densely tune them. Our implementation is in TensorFlow, and the model is trained with Tesla V100 GPUs. AdaGrad~\cite{duchi2011adaptive} with a learning rate of 0.02 is used for optimization, and we train the model for $5\times10^6$ iterations with mini-batches of size 256. The encoder operates on a single pose and is fixed for downstream tasks. More details on network architectures and model training are provided in the supplementary materials.

\begin{table*}[t]
\begin{center}
\resizebox{.82\linewidth}{!}{
\begin{tabular}{c|c|cccccc|c}
\toprule
Methods & VD & C1-R1 & C1-R2 & C2-R1 & C2-R2 & C3-R1 & C3-R2 & Average\\
\midrule
Res-TCN~\cite{kim2017interpretable} & & 40.6 / 69.6 & 39.9 / 66.8 & 30.7 / 53.3 & 48.1 / 74.5 & 48.2 / 75.5 & 29.8 / 55.3 & 39.6 / 65.8\\
ST-GCN~\cite{yan2018spatial} & & 43.3 / 73.1 & 44.1 / 72.8 & 30.7 / 57.5 & 51.4 / 79.7 & 53.1 / 82.5 & 29.7 / 59.2 & 42.1 / 70.8\\
HCN~\cite{li2018co} & & 52.5 / 80.3 & 49.8 / 76.9 & 37.2 / 63.8 & 55.5 / 86.7 & 55.3 / 83.8 & 39.4 / 69.3 & 48.3 / 76.8\\
\midrule
Auto-Encoder & \checkmark & 43.6 / 75.6 & 39.6 / 74.9 & 29.4 / 61.7 & 45.7 / 77.7 & 41.3 / 73.0 & 33.4 / 70.0 & 38.8 / 72.2\\
VAE~\cite{kingma2014auto} & \checkmark & 50.2 / 81.5 & 50.4 / 80.4 & 38.5 / 70.4 & 54.1 / 82.8 & 54.7 / 82.1 & 37.6 / 69.6 & 47.6 / 77.8\\
$\beta$-VAE~\cite{alemi2017deep,higgins2017beta} & \checkmark & 49.1 / 80.1 & 49.4 / 80.7 & 41.0 / 72.9 & 52.2 / 82.0 & 52.7 / 81.4 & 35.8 / 70.3 & 46.7 / 77.9\\
InfoNCE~\cite{oord2018representation} & \checkmark & 43.0 / 75.7 & 43.0 / 74.3 & 36.4 / 65.9 & 46.4 / 76.7 & 49.0 / 77.4 & 36.3 / 68.2 & 42.3 / 73.0\\
DIM~\cite{hjelm2019learning} & \checkmark & 42.1 / 74.4 & 42.5 / 74.3 & 32.3 / 61.8 & 45.7 / 74.0 & 43.9 / 71.3 & 33.3 / 65.6 & 40.0 / 70.2\\
Pr-VIPE~\cite{sun2019view} & & 56.1 / 85.8 & 57.9 / 85.5 & 50.7 / 84.1 & 57.3 / 84.5 & 55.9 / 84.1 & 50.9 / 83.5 & 54.8 / 84.6\\
\midrule
\textbf{CV-MIM} & \checkmark & \textbf{58.9} / \textbf{87.4} & \textbf{59.9} / \textbf{87.3} & \textbf{52.6} / \textbf{84.8} & \textbf{58.3} / \textbf{84.9} & \textbf{57.9} / \textbf{85.0} & \textbf{51.9} / \textbf{84.9} & \textbf{56.6} / \textbf{85.7}\\
\bottomrule
\end{tabular}}
\end{center}
\caption{Results of top-1 and top-5 action recognition accuracy (\%) on NTU-RGB+D~\cite{shahroudy2016ntu} with the setting of single-shot cross-view action recognition. C1, C2, and C3 are the camera identifiers; R1 and R2 are the replication numbers; one combination of them forms a unique camera view. Each time, models are trained using one view, and evaluated on all six views. We highlight view-disentangled (VD) representation learning methods. Best performances are highlighted in bold.}
\label{tbl:ntu_sscv}
\end{table*}


\begin{table}[t]
\begin{center}
\resizebox{.96\linewidth}{!}{
\begin{tabular}{c|c|ccc|cc}
\toprule
Methods & VD & RGB & Depth & Pose & CS & CV\\
\midrule
Vyas~\etal~\cite{vyas2020multi} & & \checkmark & & & 83.3 & 89.3\\
& & & \checkmark & & 70.8 & 77.5\\
Res-TCN~\cite{kim2017interpretable} & & & & \checkmark & 82.2 & 90.2\\
ST-GCN~\cite{yan2018spatial} & & & & \checkmark & 82.0 & 91.3\\
HCN~\cite{li2018co} & & & & \checkmark & 81.0 & 90.1\\
\midrule
Auto-Encoder & \checkmark & & & \checkmark & 62.5 & 71.8\\
VAE~\cite{kingma2014auto} & \checkmark & & & \checkmark & 76.8 & 88.7\\
$\beta$-VAE~\cite{alemi2017deep,higgins2017beta} & \checkmark & & & \checkmark & 76.1 & 88.8\\
InfoNCE~\cite{oord2018representation} & \checkmark & & & \checkmark & 74.1 & 82.7\\
DIM~\cite{hjelm2019learning} & \checkmark & & & \checkmark & 73.6 & 82.4\\
Pr-VIPE~\cite{sun2019view} & & & & \checkmark & \underline{77.7} & \textbf{89.7}\\
\midrule
\textbf{CV-MIM} & \checkmark & & & \checkmark & \textbf{77.8} & \underline{89.5}\\
\bottomrule
\end{tabular}}
\end{center}
\caption{Comparisons of top-1 action recognition accuracy (\%) on NTU-RGB+D~\cite{shahroudy2016ntu} with Cross-Subject (CS) and Cross-View (CV) settings. The check marks indicate the input to each method, including image pixels (RGB), depth, and model-estimated 2D pose (Pose). We highlight view-disentangled (VD) representation learning methods. Top two performances of representation learning models are highlighted in bold and underline, respectively.} 
\label{tbl:ntu_fs}
\end{table}

\subsection{Action Recognition}

We evaluate our approach in the downstream task of action recognition over a variety of settings including full-supervision, the proposed single-shot cross-view scenario, and limited-supervision.

\p{Baselines} For fully-supervised baselines taking 2D poses, we compare our approach with three state-of-the-art methods based on temporal convolutions: Res-TCN~\cite{kim2017interpretable}, ST-GCN~\cite{yan2018spatial}, and HCN~\cite{li2018co}. We also compare to other state-of-the-art methods using different input modalities.

For representation learning, generative models are commonly used for building representations. Although their target domains are different, they usually optimize the following objective based on cross reconstruction~\cite{nie2020unsupervised,peng2017reconstruction,rhodin2019neural,rhodin2018unsupervised} for view-disentangled representation learning:
\begin{equation}
\label{eq:rec_obj}
\min_{E, G} \left[\| \bm{x} - G(\bm{z}_p, \bm{z}_v)\|_2^2 + \| \bm{x} - G(\bm{z}_p^+, \bm{z}_v)\|_2^2\right],
\end{equation}
where $G$ denotes the decoder, and $\|\cdot\|_2$ denotes the $L2$ distance. In the same spirit, we implement three cross-reconstruction baselines according to Eq.~\eqref{eq:rec_obj} using auto-encoder, VAE~\cite{kingma2014auto} and $\beta$-VAE~\cite{alemi2017deep,higgins2017beta}. Their backbone networks of the encoder and decoder are both based on~\cite{martinez2017simple}. Moreover, we implement two MI-based counterparts optimizing Eq.~\eqref{eq:cvmi} through InfoNCE~\cite{oord2018representation} and DIM~\cite{hjelm2019learning}. All these representation learning approaches predict both view and pose representations like our algorithm and thus are our main competing approaches. We also include Pr-VIPE~\cite{sun2019view} for comparison since they learn view-invariant embeddings for human poses as well, but we note that they do not produce view representations.

\begin{figure}[t]
\begin{center}
   \includegraphics[width=.92\linewidth]{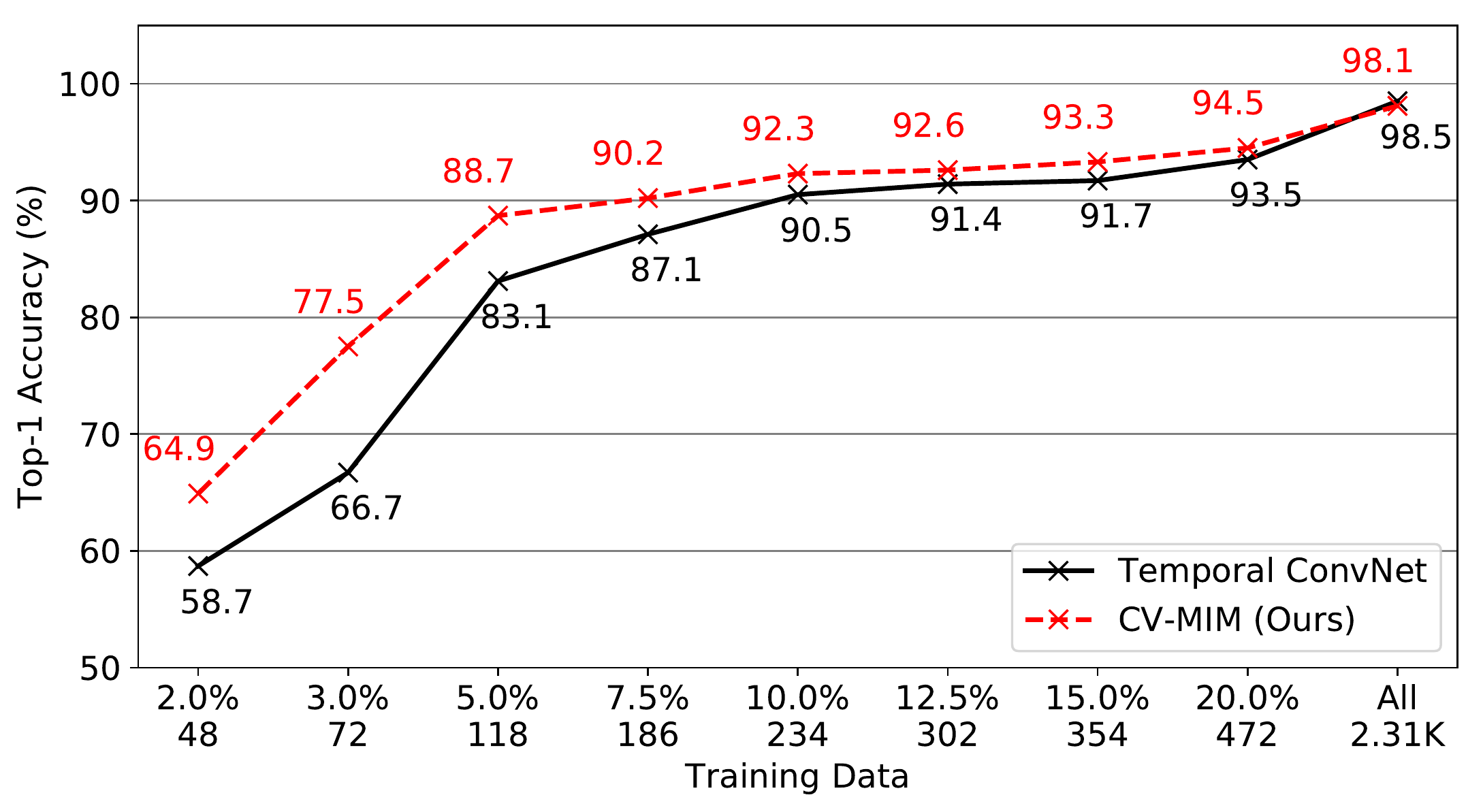}\\
   \includegraphics[width=.92\linewidth]{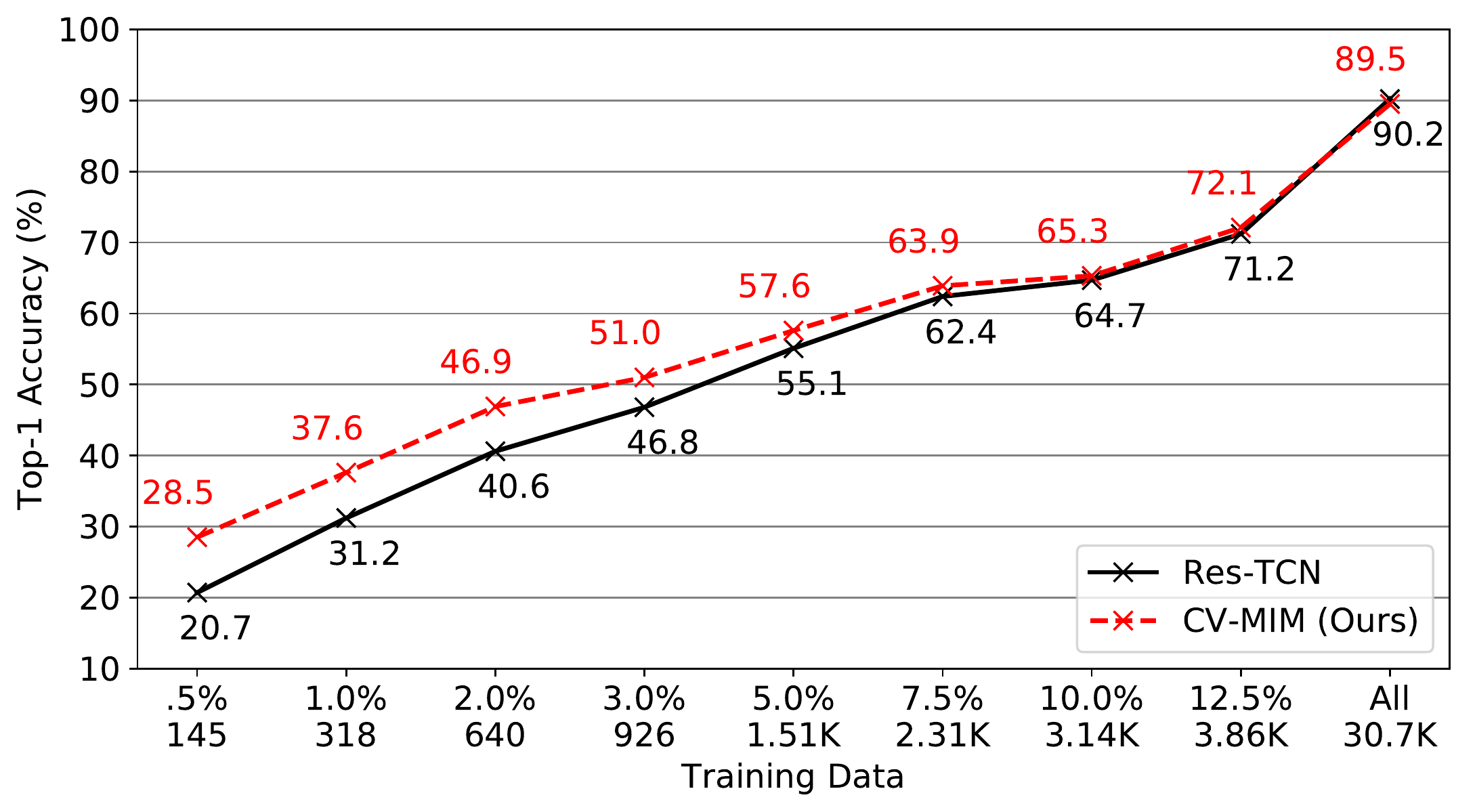}
\end{center}
\vspace{-1em}
   \caption{Recognition accuracy when limited supervisions are provided on Penn Action~\cite{zhang2013actemes} (\textbf{top}) and NTU-RGB+D~\cite{shahroudy2016ntu} (\textbf{bottom}).}
\label{fig:limit}
\end{figure}

\begin{figure*}[t]
\begin{center}
  \includegraphics[width=.94\linewidth]{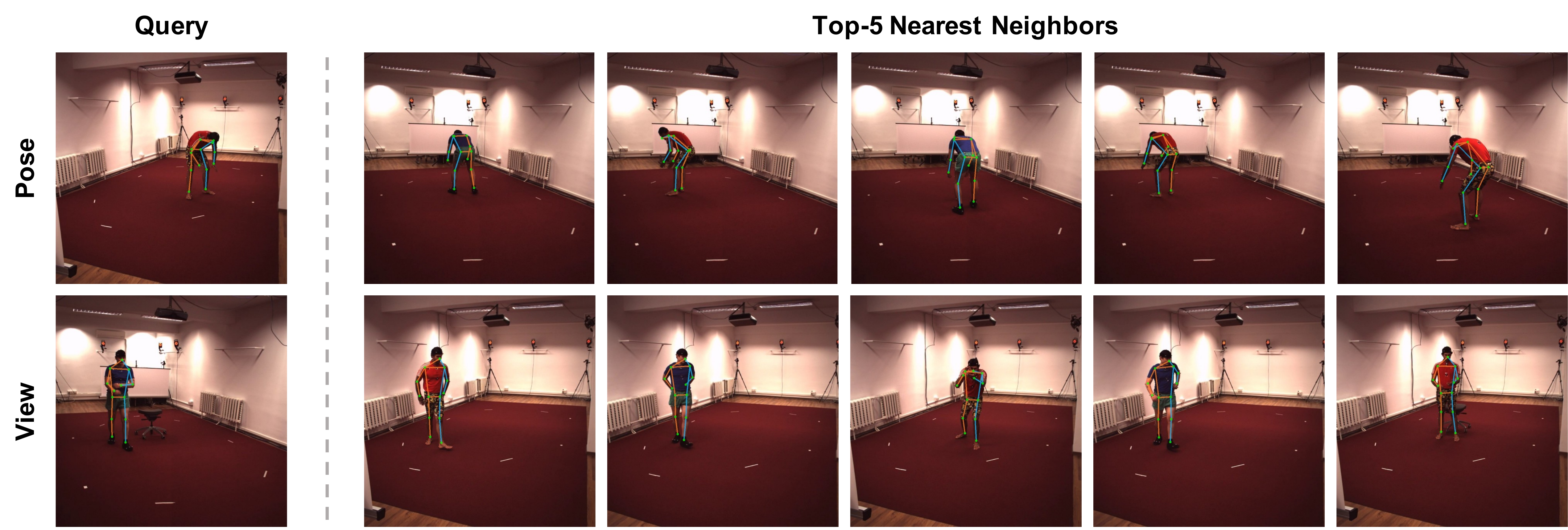}
\end{center}
\vspace{-.6em}
  \caption{Nearest neighbors in the representation space using subjects S9 and S11 on Human3.6M~\cite{ionescu2013human3}. The first row uses pose representations; the second uses view representations. On each row, we show the query on the left and its 5 nearest neighbors on the right.}
\label{fig:retrieval}
\end{figure*}

\p{Results on Penn Action} We start by evaluating our method on Penn Action~\cite{zhang2013actemes}. A simple temporal convolution network is used for aggregating temporal features from per-frame pose representations. See supplementary materials for detailed architecture and training setup. Tables~\ref{tbl:penn_sscv} and~\ref{tbl:penn_fs} show the results in the single-shot cross-view and fully-supervised settings, respectively.

We observe that our approach outperforms other view-disentangled as well as fully-supervised methods by a large margin in the single-shot setting and presents the lowest variance in accuracy. It is also worth mentioning that our results are substantially better than those baselines directly optimizing Eq.~\eqref{eq:cvmi}, which demonstrates the effectiveness of our refined formulation proposed in Eq.~\eqref{eq:full_obj}. Additionally, we match the state of the art in the fully-supervised setting and even yield better results than models using multiple input modalities.

\p{Results on NTU-RGB+D} We continue to experiment on NTU-RGB+D~\cite{shahroudy2016ntu}. The results under the same settings are reported in Tables~\ref{tbl:ntu_sscv} and~\ref{tbl:ntu_fs}, respectively. We observe that our model achieves the best performance in the single-shot setting while achieving competitive results in the fully-supervised setting. Interestingly, some cross-reconstruction models fail because of the large variances in poses and viewpoints present in this dataset. In contrast, our model is robust to these changes. 
From Table~\ref{tbl:ntu_fs}, we also observe that there is in general a considerable performance gap between the fully-supervised methods and our method in the cross-subject experiment but not in the cross-view experiment. This indicates our learned representation is not subject-invariant, which could be potentially solved by training with more subjects or augmented skeletons.

\p{Training with Limited Supervisions} We further investigate model performances when supervised data is limited. In this experiment, we use the same setup as the fully-supervised setting described above except that the amount of supervised training samples is varied. We report the results in Fig.~\ref{fig:limit}. We find that our model consistently improves the fully-supervised baselines by a notable margin when only a limited number of training samples are available. This shows that the learned pose representations can capture the semantics of 2D poses in a meaningful way which reduces the amount of supervision needed for the downstream task. We also provide additional comparisons with other representation learning methods under the same setting in the supplementary materials.

\begin{table}[t]
\begin{center}
\resizebox{.92\linewidth}{!}{
\begin{tabular}{c|ccc|c}
\toprule
Methods & Concat & Product & Mixture & Accuracy\\
\midrule
CV-MIM (full)  & \checkmark & & & 90.5\\
& & \checkmark & & 90.2\\
& & & \checkmark & \textbf{91.8}\\
\midrule
w/o $\mathcal{L}_\text{inter}$ & & & \checkmark & 86.1\\
w/o $\mathcal{L}_\text{prior}$ & & & \checkmark & 89.3\\
\bottomrule
\end{tabular}}
\end{center}
\caption{Ablation study on the fusion operation and two regularization losses used in Eq.~\eqref{eq:full_obj} on Penn Action~\cite{zhang2013actemes}.}
\label{tbl:ablation}
\end{table}

\subsection{Ablation Study}

We perform the ablative analysis to better understand the design choices of our approach from two perspectives: the fusion operation to combine view and pose representations and the effectiveness of two regularization losses used in Eq.~\eqref{eq:full_obj}. We explore three choices of the fusion operation in this experiment: concatenation, product-of-experts~\cite{wu2018multimodal}, and mixture-of-experts~\cite{shi2019variational}. Table~\ref{tbl:ablation} shows the results on Penn Action~\cite{zhang2013actemes} in the single-shot setting. We can see that our model is robust to the form of fusion operation thanks to the disentanglement loss. We select mixture-of-experts as the default fusion operation due to its highest performance. Furthermore, removing either regularization loss in Eq.~\eqref{eq:full_obj} results in a significant performance downgrade, which demonstrates their effectiveness.

\subsection{Qualitative Results} 

Last but not least, we show qualitative results when using the learned representations of our model for nearest neighbor retrieval. In Fig.~\ref{fig:retrieval}, we show that our pose representations can successfully find similar poses from different views in the testing set of Human3.6M~\cite{ionescu2013human3}. Interestingly, we also show that the learned view representations can retrieve frames captured from a viewpoint that is very similar to the query while the poses are different. More visual results are provided in the supplementary materials. 

%% file: 5_conclusion.tex
\section{Conclusion}

We present CV-MIM, a representation learning approach to encode both pose-dependent and view-dependent representations for 2D human poses by maximizing cross-view MI. 
We further motivate two regularization losses to encourage disentanglement and smoothness of the learned representations from a theoretical perspective. 
We show that using the learned pose representation achieves significant improvement from existing representation learning methods on downstream action recognition tasks, and even outperforms fully-supervised baselines in many settings.
We also demonstrate the learned view representation can be directly applied to similar view retrieval among different poses.
CV-MIM focuses on single person representation, and for future work, we will investigate its multi-person extension.



%% file: 6_supp_theoretical_results.tex
\section{Theoretical Results}

\subsection{Proofs}

This section provides detailed proofs on how Eq.~(2) is simplified to Eq.~(5) in the main paper. We begin by introducing the following two propositions on MI.

\begin{proposition}
Given any three random variables $\bm{x}$, $\bm{y}$ and $\bm{z}$, with a joint distribution $p(\bm{x}, \bm{y}, \bm{z})$. If $\bm{y}$ and $\bm{z}$ are independent and conditionally independent, i.e., $p(\bm{y}, \bm{z}) = p(\bm{y})p(\bm{z})$ and $p(\bm{y}, \bm{z}|\bm{x}) = p(\bm{y}|\bm{x})p(\bm{z}|\bm{x})$, we have that,
\begin{equation*}
\mathcal{I}(\bm{x}; \bm{y}, \bm{z}) = \mathcal{I}(\bm{x}; \bm{y}) + \mathcal{I}(\bm{x}; \bm{z}).
\end{equation*}
\label{supp:prop:dis}
\end{proposition}

\begin{proof}
We start by applying the chain rule to $\mathcal{I}(\bm{x}; \bm{y}, \bm{z})$:
\begin{equation*}
\mathcal{I}(\bm{x}; \bm{y}, \bm{z}) = \mathcal{I}(\bm{x}; \bm{y}|\bm{z}) + \mathcal{I}(\bm{x}; \bm{z}).
\end{equation*}
Then we have,
\begin{align*}
\mathcal{I}(\bm{x}; \bm{y}|\bm{z}) &= \sum_{\bm{z}} \sum_{\bm{y}} \sum_{\bm{x}} p(\bm{x}, \bm{y}, \bm{z}) \log \frac{p(\bm{x}, \bm{y}, \bm{z})p(\bm{z})}{p(\bm{x}, \bm{z})p(\bm{y}, \bm{z})}\\
&=\sum_{\bm{z}} \sum_{\bm{y}} \sum_{\bm{x}} p(\bm{x}, \bm{y}, \bm{z}) \log \frac{p(\bm{x}|\bm{y}, \bm{z})}{p(\bm{x}|\bm{z})},
\end{align*}
and,
\begin{align*}
\mathcal{I}(\bm{x}; \bm{z}) &= \sum_{\bm{z}} \sum_{\bm{x}} p(\bm{x}, \bm{z}) \log \frac{p(\bm{x}, \bm{z})}{p(\bm{x})p(\bm{z})}\\
&=\sum_{\bm{z}} \sum_{\bm{x}} p(\bm{x})p(\bm{z}|\bm{x}) \log \frac{p(\bm{x}|\bm{z})p(\bm{z})}{p(\bm{x})p(\bm{z})}\\
&=\sum_{\bm{z}} \sum_{\bm{x}} p(\bm{x})p(\bm{z}|\bm{x}) \log \frac{p(\bm{x}|\bm{z})}{p(\bm{x})}\\
&=\sum_{\bm{z}} \sum_{\bm{y}} \sum_{\bm{x}} p(\bm{x})p(\bm{y}|\bm{x})p(\bm{z}|\bm{x}) \log \frac{p(\bm{x}|\bm{z})}{p(\bm{x})}\\
&=\sum_{\bm{z}} \sum_{\bm{y}} \sum_{\bm{x}} p(\bm{x})p(\bm{y}, \bm{z}|\bm{x}) \log \frac{p(\bm{x}|\bm{z})}{p(\bm{x})}\\
&=\sum_{\bm{z}} \sum_{\bm{y}} \sum_{\bm{x}} p(\bm{x}, \bm{y}, \bm{z}) \log \frac{p(\bm{x}|\bm{z})}{p(\bm{x})}.
\end{align*}
After combining them together, we can show that,
\begin{align*}
\mathcal{I}(\bm{x}; \bm{y}, \bm{z}) &= \sum_{\bm{z}} \sum_{\bm{y}} \sum_{\bm{x}} p(\bm{x}, \bm{y}, \bm{z}) \log \frac{p(\bm{x}|\bm{y}, \bm{z})}{p(\bm{x})}\\
&= \sum_{\bm{z}} \sum_{\bm{y}} \sum_{\bm{x}} p(\bm{x}, \bm{y}, \bm{z}) \log \frac{p(\bm{y}, \bm{z}|\bm{x})}{p(\bm{y}, \bm{z})}\\
&= \sum_{\bm{z}} \sum_{\bm{y}} \sum_{\bm{x}} p(\bm{x}, \bm{y}, \bm{z}) \log \frac{p(\bm{y}|\bm{x})p(\bm{z}|\bm{x})}{p(\bm{y})p(\bm{z})}\\
&= \sum_{\bm{z}} \sum_{\bm{y}} \sum_{\bm{x}} p(\bm{x}, \bm{y}, \bm{z}) \log \left[ \frac{p(\bm{y}|\bm{x})}{p(\bm{y})}\frac{p(\bm{z}|\bm{x})}{p(\bm{z})}\right].
\end{align*}
Hence, we can rewrite $\mathcal{I}(\bm{x}; \bm{y}, \bm{z})$ as:
\begin{multline*}
\mathcal{I}(\bm{x}; \bm{y}, \bm{z}) = \sum_{\bm{z}} \sum_{\bm{y}} \sum_{\bm{x}} p(\bm{x}, \bm{y}, \bm{z}) \log \frac{p(\bm{y}|\bm{x})}{p(\bm{y})}\\ + \sum_{\bm{z}} \sum_{\bm{y}} \sum_{\bm{x}} p(\bm{x}, \bm{y}, \bm{z}) \log \frac{p(\bm{z}|\bm{x})}{p(\bm{z})}.
\end{multline*}
According to the definition of MI, we have that,
\begin{align*}
&\sum_{\bm{z}} \sum_{\bm{y}} \sum_{\bm{x}} p(\bm{x}, \bm{y}, \bm{z}) \log \frac{p(\bm{y}|\bm{x})}{p(\bm{y})}\\
=& \sum_{\bm{z}} \sum_{\bm{y}} \sum_{\bm{x}} p(\bm{x})p(\bm{y}|\bm{x})p(\bm{z}|\bm{x}) \log \frac{p(\bm{y}|\bm{x})}{p(\bm{y})}\\
=& \sum_{\bm{y}} \sum_{\bm{x}} p(\bm{x}, \bm{y}) \log \frac{p(\bm{y}|\bm{x})}{p(\bm{y})} = \mathcal{I}(\bm{x}; \bm{y}),
\end{align*}
and,
\begin{align*}
&\sum_{\bm{z}} \sum_{\bm{y}} \sum_{\bm{x}} p(\bm{x}, \bm{y}, \bm{z}) \log \frac{p(\bm{z}|\bm{x})}{p(\bm{z})}\\
=& \sum_{\bm{z}} \sum_{\bm{y}} \sum_{\bm{x}} p(\bm{x})p(\bm{y}|\bm{x})p(\bm{z}|\bm{x}) \log \frac{p(\bm{z}|\bm{x})}{p(\bm{z})}\\
=& \sum_{\bm{z}} \sum_{\bm{x}} p(\bm{x}, \bm{z}) \log \frac{p(\bm{z}|\bm{x})}{p(\bm{z})} = \mathcal{I}(\bm{x}; \bm{z}),
\end{align*}
which prove the proposition.
\end{proof}

The next proposition is a minor adaptation of~\cite{shwartz2017opening} according to Data Processing Inequality (DPI)~\cite{cover2012elements}.

\begin{proposition}[Shwartz \& Tishby~\cite{shwartz2017opening}]
Given any two random variables $\bm{x}$ and $\bm{y}$, and any representation variable $\bm{z}$, defined as a (possibly stochastic) map of the input $\bm{x}$, we have the following DPI chain:
\begin{equation*}
\mathcal{I}(\bm{x}; \bm{y}) \geqslant \mathcal{I}(\bm{z}; \bm{y}).
\end{equation*}
\label{supp:prop:dpi}
\end{proposition}

As defined in the main paper, we let $\bm{x}^i$ denote the given 2D pose from the $i$-th view. We are interested in learning an encoding network $E$ that produces a view representation $\bm{z}_v^i$, and a pose representation $\bm{z}_p^i$ from the input $\bm{x}^i$. For simplicity, we define an optimal intermediate representation $\bm{z}^i$ that satisfies $p(\cdot, \bm{z}^i) = p(\cdot, \bm{z}_p^i, \bm{z}_v^i)$, and we have $\mathcal{I}(\cdot; \bm{z}^i) = \mathcal{I}(\cdot; \bm{z}_p^i, \bm{z}_v^i)$. Therefore, based on Proposition~\ref{supp:prop:dis}, the following equation holds:
\begin{equation*}
\mathcal{I}(\bm{x}^i; \bm{z}_p^j, \bm{z}_v^i) = \mathcal{I}(\bm{x}^i; \bm{z}_p^j) + \mathcal{I}(\bm{x}^i; \bm{z}_v^i).
\end{equation*}
According to Proposition~\ref{supp:prop:dpi}, we have that,
\begin{equation*}
\mathcal{I}(\bm{x}^i; \bm{z}_p^j) \geqslant \mathcal{I}(\bm{z}^i; \bm{z}_p^j) = \mathcal{I}(\bm{z}_p^i, \bm{z}_v^i; \bm{z}_p^j).
\end{equation*}
After applying Proposition~\ref{supp:prop:dis}, we have that,
\begin{equation*}
\mathcal{I}(\bm{z}_p^i, \bm{z}_v^i; \bm{z}_p^j) = \mathcal{I}(\bm{z}_p^i; \bm{z}_p^j) + \mathcal{I}(\bm{z}_v^i; \bm{z}_p^j).
\end{equation*}
As $\bm{z}_v^i$ and $\bm{z}_p^j$ are independent (mutually exclusive), it holds that $\mathcal{I}(\bm{z}_v^i; \bm{z}_p^j) = 0$. Therefore, we have that,
\begin{equation*}
\mathcal{I}(\bm{x}^i; \bm{z}_p^j) \geqslant \mathcal{I}(\bm{z}_p^i, \bm{z}_v^i; \bm{z}_p^j) = \mathcal{I}(\bm{z}_p^i; \bm{z}_p^j).
\end{equation*}
Similarly, we have that,
\begin{align*}
\mathcal{I}(\bm{x}^i; \bm{z}_v^i) &\geqslant \mathcal{I}(\bm{z}^i; \bm{z}_v^i) = \mathcal{I}(\bm{z}_p^i, \bm{z}_v^i; \bm{z}_v^i)\\
&= \mathcal{I}(\bm{z}_p^i; \bm{z}_v^i) + \mathcal{I}(\bm{z}_v^i; \bm{z}_v^i)\\
&= \mathcal{I}(\bm{z}_p^i; \bm{z}_v^i) + \mathcal{I}(\bm{z}_v^i; \bm{z}_v^i)\\
&= \mathcal{H}(\bm{z}_v^i),
\end{align*}
where $\mathcal{H}$ is the Shannon entropy which is always non-negative, \ie, $\mathcal{H} \geqslant 0$. Then, we have that,
\begin{equation*}
\mathcal{I}(\bm{x}^i; \bm{z}_p^j, \bm{z}_v^i) \geqslant \mathcal{I}(\bm{z}_p^i; \bm{z}_p^j) + \mathcal{H}(\bm{z}_v^i) \geqslant \mathcal{I}(\bm{z}_p^i; \bm{z}_p^j).
\end{equation*}
Finally, we can obtain that,
\begin{align*}
&\sum_{i} \mathcal{I}(\bm{x}^i; \bm{z}_p^i, \bm{z}_v^i) + \sum_{i \neq j} \mathcal{I}(\bm{x}^i; \bm{z}^j_p, \bm{z}^i_v)\\
\geqslant& \sum_{i} \mathcal{I}(\bm{x}^i; \bm{z}_p^i, \bm{z}_v^i) + \sum_{i \neq j} \mathcal{I}(\bm{z}^i_p; \bm{z}^j_p),
\end{align*}
which demonstrates the claim.

\subsection{Relation to Cross Reconstruction}

In this section, we provide an intuitive explanation on the relationship between the proposed cross-view MI maximization and the conventional methods based on cross reconstruction~\cite{nie2020unsupervised,peng2017reconstruction,rhodin2019neural,rhodin2018unsupervised}. According to~\cite{hjelm2019learning}, in the context of representation learning, given the input $\bm{x}$ and its representation $\bm{z}$ encoded by a neural network, the reconstruction error can be related to the MI as follows:
\begin{align*}
\mathcal{I}(\bm{x}; \bm{z}) &= \mathcal{H}(\bm{x}) - \mathcal{H}(\bm{x}|\bm{z}) \geqslant - \mathcal{H}(\bm{x}|\bm{z})\\
&= \sum_{\bm{z}} \sum_{\bm{x}} p(\bm{x}, \bm{z}) \log p(\bm{x}|\bm{z})\\
&= \sum_{\bm{z}} \sum_{\bm{x}} p(\bm{x})p(\bm{z}|\bm{x}) \log p(\bm{x}|\bm{z})\\
&\geqslant \sum_{\bm{x}} p(\bm{x}) \sum_{\bm{z}} p(\bm{z}|\bm{x}) \log p(\bm{x}|\bm{z})\\
&= \mathbb{E}_{\bm{x} \sim p(\bm{x})}\left\{ \mathbb{E}_{\bm{z} \sim p(\bm{z}|\bm{x})} \left[ \log p(\bm{x}|\bm{z}) \right] \right\},
\end{align*}
where the inequality in the second-to-last line is achieved according to Jensen's inequality if we assume that the probability density functions are convex.

In typical settings of reconstruction, $p(\bm{z}|\bm{x})$ can be interpreted as an encoder while $p(\bm{x}|\bm{z})$ is the decoder. For example, the Variational Auto-Encoders (VAEs)~\cite{kingma2014auto} approximate $p(\bm{z}|\bm{x})$ by a tractable variational distribution $q(\bm{z}|\bm{x})$ with the KL divergence term $\mathbb{D}_\text{KL}[q(\bm{z}|\bm{x}) \| p(\bm{z})]$ which ensures that the learned distribution $q$ is similar to the true prior distribution. To maximize the above objective, reconstruction-type methods usually minimize the mean squared error between the input and reconstruction if a Gaussian distribution is assumed or binary cross-entropy loss if a Bernoulli distribution is assumed. Therefore, intuitively, the reconstruction-type objective is a lower bound of MI, and similar conclusions exist for other generative models based on reconstruction~\cite{alemi2017deep,higgins2017beta}.

We can easily extend the above formulation to the proposed cross-view MI maximization setup in the main paper, where we have that,
\begin{align*}
&\sum_{i} \mathcal{I}(\bm{x}^i; \bm{z}_p^i, \bm{z}_v^i) + \sum_{i \neq j} \mathcal{I}(\bm{x}^i; \bm{z}^j_p, \bm{z}^i_v)\\
\geqslant& \sum_{i} \mathbb{E}_{p(\bm{x}^i)}\left\{ \mathbb{E}_{p(\bm{z}_p^i, \bm{z}_v^i|\bm{x}^i)} \left[ \log p(\bm{x}^i|\bm{z}_p^i, \bm{z}_v^i) \right] \right\}\\
&+ \sum_{i \neq j} \mathbb{E}_{p(\bm{x}^i, \bm{x}^j)}\left\{ \mathbb{E}_{p(\bm{z}_p^j, \bm{z}_v^i|\bm{x}^i, \bm{x}^j)} \left[ \log p(\bm{x}^i|\bm{z}_p^j, \bm{z}_v^i) \right] \right\},
\end{align*}
where we can see that in an approximate sense, the existing methods for view-disentangled representation learning~\cite{nie2020unsupervised,peng2017reconstruction,rhodin2019neural,rhodin2018unsupervised} based on cross-reconstruction maximize a lower bound of the proposed cross-view MI maximization.

%% file: 7_supp_implementation_details.tex
\section{Implementation Details}

\subsection{Representation Learning}

The backbone network architecture for our encoding network $E$ is based on~\cite{martinez2017simple}. We use two residual blocks, batch normalization, 0.25 dropout, and no maximum weight norm constraint~\cite{martinez2017simple}. Both the likelihood estimation network $Q$ and discriminator $D$ are implemented by multi-layer perceptrons (MLPs). To be specific, $Q$ consists of two fully-connected layers where the first layer is followed by the batch normalization and ELU activation~\cite{clevert2016fast}; $D$ contains three fully-connected layers where the first two layers are followed by the ReLU activation. All these networks are trained using AdaGrad~\cite{duchi2011adaptive} with a fixed learning rate of 0.02 for optimization. For fair comparisons, all the representation learning methods compared to in the experiments also use the same architecture and training setup as described here. We also note that the learned representations are fixed during downstream training.

\subsection{Action Recognition}

\begin{table}[t]
\begin{center}
\resizebox{.92\linewidth}{!}{
\begin{tabular}{c|c|c}
\toprule
Layer & Output Size & Setting\\
\midrule
Conv1D & $166 \times 64$ & $1 \times 7$, stride 2, BN, ReLU, 0.5\\
Conv1D & $83 \times 128$ & $1 \times 7$, stride 2, BN, ReLU, 0.5\\
Conv1D & $42 \times 256$ & $1 \times 7$, stride 2, BN, ReLU, 0.5\\
\midrule
Pooling & $256$ & global average pooling\\
Dense & $14$ & SoftMax\\
\bottomrule
\end{tabular}}
\end{center}
\caption{Architecture used on Penn Action~\cite{zhang2013actemes}. The setting of ``$1 \times 7$, stride 2, BN, ReLU, 0.5'' refers to 1D Convolution with kernel size of $1 \times 7$ and stride 2 followed by batch normalization (BN), ReLU activation, and a dropout layer with 0.5 drop rate.}
\label{tbl:supp:arch_penn}
\end{table}

\begin{table}[t]
\begin{center}
\resizebox{.96\linewidth}{!}{
\begin{tabular}{c|c|c}
\toprule
Layer & Output Size & Setting\\
\midrule
Conv1D & $300 \times 64$ & $1 \times 7$, stride 1, BN, ReLU\\
\midrule
R-Conv1D & & $1 \times 7$, stride 2, BN, ReLU, 0.5\\
& & $1 \times 5$, stride 1, BN, ReLU, 0.5\\
& & $1 \times 3$, stride 1, BN\\
Shortcut & $150 \times 64$ & $1 \times 1$, stride 2, BN, ReLU\\
\midrule
R-Conv1D & & $1 \times 7$, stride 2, BN, ReLU, 0.5\\
& & $1 \times 5$, stride 1, BN, ReLU, 0.5\\
& & $1 \times 3$, stride 1, BN\\
Shortcut & $75 \times 128$ & $1 \times 1$, stride 2, BN, ReLU\\
\midrule
R-Conv1D & & $1 \times 7$, stride 1, BN, ReLU, 0.5\\
& & $1 \times 5$, stride 1, BN, ReLU, 0.5\\
& & $1 \times 3$, stride 1, BN\\
Shortcut & $75 \times 256$ & $1 \times 1$, stride 1, BN, ReLU\\
\midrule
Pooling & $256$ & global average pooling\\
Dense & $49$ & SoftMax\\
\bottomrule
\end{tabular}}
\end{center}
\caption{Architecture used on NTU-RGB+D~\cite{shahroudy2016ntu}. R-Conv1D represents the residual layer introduced in~\cite{he2016deep}. The setting of ``$1 \times 7$, stride 2, BN, ReLU, 0.5'' refers to 1D Convolution with kernel size of $1 \times 7$ and stride 2 followed by batch normalization (BN), ReLU activation, and a dropout layer with 0.5 drop rate.}
\label{tbl:supp:arch_ntu}
\end{table}

\p{Penn Action} We use a simple temporal convolution network to extract temporal features from per-frame pose representations generated by the encoding network. Table~\ref{tbl:supp:arch_penn} presents the architecture of this network. We use Adam~\cite{kingma2014adam} with a fixed learning rate of $1.0 \times 10^{-5}$ for optimization. We set the size of mini-batches to 64 and the network is trained for $1 \times 10^6$ iterations. During network training, we perform data augmentation by randomly horizontally flipping all frames in a video.

\p{NTU-RGB+D} We use ResNet1D which is a modified version of~\cite{he2016deep} as the backbone network for action recognition on this dataset. Its detailed network architecture is shown in Table~\ref{tbl:supp:arch_ntu}. We use Adam~\cite{kingma2014adam} with a fixed learning rate of $1.0 \times 10^{-3}$ for training the network. We set the size of mini-batches to 64 and the network is optimized for $1 \times 10^6$ iterations. Following the practice of~\cite{yan2018spatial}, no data augmentation is performed during training on this dataset. We only use single-person action categories, including action labels from A1 to A49, in this experiment.

%% file: 8_supp_additional_results.tex
\section{Additional Results}

\p{Effectiveness of Camera Augmentation} We evaluate the effectiveness of camera augmentation by training a variant of CV-MIM where camera augmentation is not utilized. As shown in Table~\ref{tbl:supp:ca}, training with camera augmentation leads to around 1\% and 2\% performance improvements on Penn Action and NTU-RGB+D, respectively.

\p{More Results on NTU-RGB+D} We report the classification accuracy and standard deviation of different models on NTU-RGB+D with the setting of single-shot cross-view action recognition averaged over five repeated runs. Table~\ref{tbl:supp:ntu_sscv} shows the results. We can see that our method achieves the best mean accuracy and smallest standard deviations.

\begin{table}[t]
\begin{center}
\resizebox{.9\linewidth}{!}{
\begin{tabular}{c|cc}
\toprule
Methods & Penn Action~\cite{zhang2013actemes} & NTU-RGB+D~\cite{shahroudy2016ntu}\\
\midrule
CV-MIM & \textbf{91.75 $\pm$ 0.24} & \textbf{56.50 $\pm$ 0.13}\\
CV-MIM w/o CA & 90.75 $\pm$ 0.30 & 54.50 $\pm$ 0.22\\
\bottomrule
\end{tabular}}
\end{center}
\caption{Classification accuracy (\%) and standard deviation of CV-MIM with or without camera augmentation (CA) on Penn Action~\cite{zhang2013actemes} and NTU-RGB+D~\cite{shahroudy2016ntu} with the setting of single-shot cross-view action recognition.}
\label{tbl:supp:ca}
\end{table}

\begin{table*}[t]
\begin{center}
\resizebox{.94\linewidth}{!}{
\begin{tabular}{c|c|cccccc|c}
\toprule
Methods & VD & C1-R1 & C1-R2 & C2-R1 & C2-R2 & C3-R1 & C3-R2 & Average\\
\midrule
Res-TCN~\cite{kim2017interpretable} & & 40.78 $\pm$ 0.39 & 39.71 $\pm$ 0.66 & 30.02 $\pm$ 0.63 & 48.45 $\pm$ 0.37 & 48.90 $\pm$ 0.42 & 29.25 $\pm$ 0.37 & 39.52 $\pm$ 0.23\\
\midrule
Auto-Encoder & \checkmark & 42.76 $\pm$ 0.93 & 40.79 $\pm$ 0.70 & 32.13 $\pm$ 1.77 & 44.86 $\pm$ 0.59 & 42.39 $\pm$ 2.60 & 32.38 $\pm$ 1.34 & 39.22 $\pm$ 0.58\\
VAE~\cite{kingma2014auto} & \checkmark & 49.96 $\pm$ 0.61 & 50.92 $\pm$ 0.61 & 38.50 $\pm$ 0.66 & 53.63 $\pm$ 0.30 & 54.40 $\pm$ 0.49 & 36.94 $\pm$ 0.50 & 47.39 $\pm$ 0.26\\
$\beta$-VAE~\cite{alemi2017deep,higgins2017beta} & \checkmark & 49.06 $\pm$ 1.07 & 48.81 $\pm$ 0.67 & 40.58 $\pm$ 0.84 & 51.39 $\pm$ 0.76 & 52.09 $\pm$ 0.44 & 36.87 $\pm$ 0.60 & 46.47 $\pm$ 0.21\\
InfoNCE~\cite{oord2018representation} & \checkmark & 43.11 $\pm$ 0.32 & 42.60 $\pm$ 0.35 & 35.11 $\pm$ 0.88 & 47.01 $\pm$ 0.47 & 48.15 $\pm$ 0.69 & 34.37 $\pm$ 1.65 & 41.72 $\pm$ 0.37\\
DIM~\cite{hjelm2019learning} & \checkmark & 41.71 $\pm$ 0.43 & 42.03 $\pm$ 0.80 & 32.63 $\pm$ 0.28 & 44.65 $\pm$ 0.70 & 43.75 $\pm$ 0.37 & 33.19 $\pm$ 0.51 & 39.66 $\pm$ 0.22\\
Pr-VIPE~\cite{sun2019view} & & 56.28 $\pm$ 0.22 & 56.95 $\pm$ 0.58 & 50.09 $\pm$ 0.73 & 50.38 $\pm$ 0.39 & 57.03 $\pm$ 0.32 & \textbf{55.41 $\pm$ 0.38} & 54.36 $\pm$ 0.28\\
\midrule
\textbf{CV-MIM} & \checkmark & \textbf{58.61 $\pm$ 0.31} & \textbf{59.67 $\pm$ 0.29} & \textbf{52.53 $\pm$ 0.32} & \textbf{58.34 $\pm$ 0.25} & \textbf{57.79 $\pm$ 0.31} & 52.08 $\pm$ 0.37 & \textbf{56.50 $\pm$ 0.13}\\
\bottomrule
\end{tabular}}
\end{center}
\caption{Classification accuracy (\%) and standard deviation of models on NTU-RGB+D~\cite{shahroudy2016ntu} with the setting of single-shot cross-view action recognition. C1, C2, and C3 are the camera identifiers; R1 and R2 are the replication numbers; one combination of them forms a unique camera view. Each time, models are trained using one view, and evaluated on all six views. We highlight view-disentangled (VD) representation learning methods. Best performances are highlighted in bold.}
\label{tbl:supp:ntu_sscv}
\end{table*}

\begin{table*}[t]
\begin{center}
\resizebox{.98\linewidth}{!}{
\begin{tabular}{c|c|ccccccccc}
\toprule
Methods & VD & 2.0\% (48) & 3.0\% (72) & 5.0\% (118) & 7.5\% (186) & 10.0\% (234) & 12.5\% (302) & 15.0\% (354) & 20.0\% (472) & All (2.31K)\\
\midrule
Temporal ConvNet & & 58.7 & 66.7 & 83.1 & 87.1 & 90.5 & 91.4 & 91.7 & 93.5 & \textbf{98.5}\\
\midrule
Auto-Encoder & \checkmark & 59.8 & 72.2 & 83.6 & 88.2 & 90.6 & 91.8 & 92.9 & 94.1 & 97.7\\
VAE~\cite{kingma2014auto} & \checkmark & 60.6 & 71.7 & 81.9 & 87.4 & 91.1 & 91.7 & 92.9 & 94.0 & 97.6\\
$\beta$-VAE~\cite{alemi2017deep,higgins2017beta} & \checkmark & 61.2 & 71.0 & 79.2 & 88.3 & 89.9 & 90.1 & 92.4 & 94.4 & 97.7\\
InfoNCE~\cite{oord2018representation} & \checkmark & 59.8 & 69.8 & 81.1 & 85.2 & 90.6 & 90.9 & 92.3 & 93.9 & 97.5\\
DIM~\cite{hjelm2019learning} & \checkmark & 57.6 & 68.1 & 80.0 & 84.3 & 87.7 & 90.0 & 90.6 & 94.3 & 97.3\\
\midrule
\textbf{CV-MIM} & \checkmark & \textbf{64.9} & \textbf{77.5} & \textbf{88.7} & \textbf{90.2} & \textbf{92.3} & \textbf{92.6} & \textbf{93.3} & \textbf{94.5} & 98.1\\
\bottomrule
\end{tabular}}
\end{center}
\caption{Comparisons of classification accuracy (\%) on Penn Action~\cite{zhang2013actemes} with the fully-supervised setting when the amount of training samples is varied. We highlight view-disentangled (VD) representation learning methods. Best performances are highlighted in bold.}
\label{tbl:supp:ls_penn}
\end{table*}

\begin{table*}[t]
\begin{center}
\resizebox{.98\linewidth}{!}{
\begin{tabular}{c|c|ccccccccc}
\toprule
Methods & VD & .5\% (145) & 1.0\% (318) & 2.0\% (640) & 3.0\% (926) & 5.0\% (1.51K) & 7.5\% (2.31K) & 10.0\% (3.4K) & 12.5\% (3.86K) & All (30.7K)\\
\midrule
Res-TCN~\cite{kim2017interpretable} & & 20.7 & 31.2 & 40.6 & 46.8 & 55.1 & 62.4 & 64.7 & 71.2 & \textbf{90.2}\\
\midrule
Auto-Encoder & \checkmark & 22.4 & 27.5 & 40.5 & 45.6 & 55.4 & 59.2 & 64.5 & 67.1 & 71.8\\
VAE~\cite{kingma2014auto} & \checkmark & 22.9 & 31.6 & 42.3 & 49.4 & 55.0 & 61.9 & 64.4 & 69.4 & 88.7\\
$\beta$-VAE~\cite{alemi2017deep,higgins2017beta} & \checkmark & 21.2 & 30.2 & 42.2 & 50.7 & 55.4 & 61.8 & 64.7 & 69.5 & 88.8\\
InfoNCE~\cite{oord2018representation} & \checkmark & 20.8 & 26.5 & 37.8 & 43.0 & 49.8 & 55.1 & 58.3 & 62.1 & 82.7\\
DIM~\cite{hjelm2019learning} & \checkmark & 20.2 & 23.8 & 35.0 & 40.0 & 47.6 & 52.8 & 57.0 & 59.9 & 82.4\\
\midrule
\textbf{CV-MIM} & \checkmark & \textbf{28.5} & \textbf{37.6} & \textbf{46.9} & \textbf{51.0} & \textbf{57.6} & \textbf{63.9} & \textbf{65.3} & \textbf{72.1} & 89.5\\
\bottomrule
\end{tabular}}
\end{center}
\caption{Results of top-1 classification accuracy (\%) on NTU-RGB+D~\cite{shahroudy2016ntu} with the cross-view benchmark when the amount of training samples is varied. We highlight view-disentangled (VD) representation learning methods. Best performances are highlighted in bold.}
\label{tbl:supp:ls_ntu}
\end{table*}

\p{More Results with Limited-Supervision} We provide additional comparisons with view-disentangled representation learning baselines under limited-supervision. Results on Penn Action~\cite{zhang2013actemes} and NTU-RGB+D~\cite{shahroudy2016ntu} are reported in Tables~\ref{tbl:supp:ls_penn} and~\ref{tbl:supp:ls_ntu}, respectively. We can see that the proposed CV-MIM outperforms other methods by a large margin consistently under different ratios of training data.

\p{View Classification} We explore the utility of learned view representations by applying them to the task of view classification on Penn Action~\cite{zhang2013actemes}. In this experiment, our target is to predict the view category, \ie, left, right, front, or back, for each frame in a video. This is achieved by training a linear classifier which takes the learned view representations as input and it is trained by the ground truth view labels provided by this dataset. We use AdaGrad~\cite{duchi2011adaptive} with a fixed learning rate of $1.0 \times 10^{-2}$ for training the classifier. We set the size of mini-batches to 64 and the classifier is optimized for $1 \times 10^4$ iterations. The data split of the fully-supervised setting is utilized to train and evaluate all view-disentangled representation learning approaches. We also compare our method to a baseline which directly takes raw 2D poses as input.

\begin{table*}[t]
\begin{center}
\resizebox{.7\linewidth}{!}{
\begin{tabular}{c|c|ccccc|c}
\toprule
Methods & 2D Pose & Auto-Encoder & VAE~\cite{kingma2014auto} & $\beta$-VAE~\cite{alemi2017deep,higgins2017beta} & InfoNCE~\cite{oord2018representation} & DIM~\cite{hjelm2019learning} & \textbf{CV-MIM}\\
\midrule
Accuracy (\%) & 62.13 & 66.24 & 66.49 & 65.63 & 66.14 & 65.71 & \textbf{67.28}\\
\bottomrule
\end{tabular}}
\end{center}
\caption{Results of view classification on Penn Action~\cite{zhang2013actemes}. Best performances are highlighted in bold.}
\label{tbl:supp:view}
\end{table*}

Table~\ref{tbl:supp:view} shows the results of each method. We observe that our model obtains the best performance among representation learning methods, and we also outperform the baseline taking 2D poses. These results demonstrate that our learned view representations manage to encode effective view information for 2D poses, and can serve as a strong model for view-relevant downstream tasks. 

\p{More Visual Results} We show more qualitative results when using the learned representations of our model for nearest neighbor retrieval on Human3.6M~\cite{ionescu2013human3} in Fig.~\ref{fig:supp:h36m}. We also show additional nearest neighbor retrieval results when applying our model on MPI-INF-3DHP~\cite{mehta2017monocular} in Fig.~\ref{fig:supp:3dhp}. Interestingly, we find that our learned representations are able to generalize to new views and new poses contained in MPI-INF-3DHP.

\begin{figure*}[t]
\begin{center}
  \includegraphics[width=\linewidth]{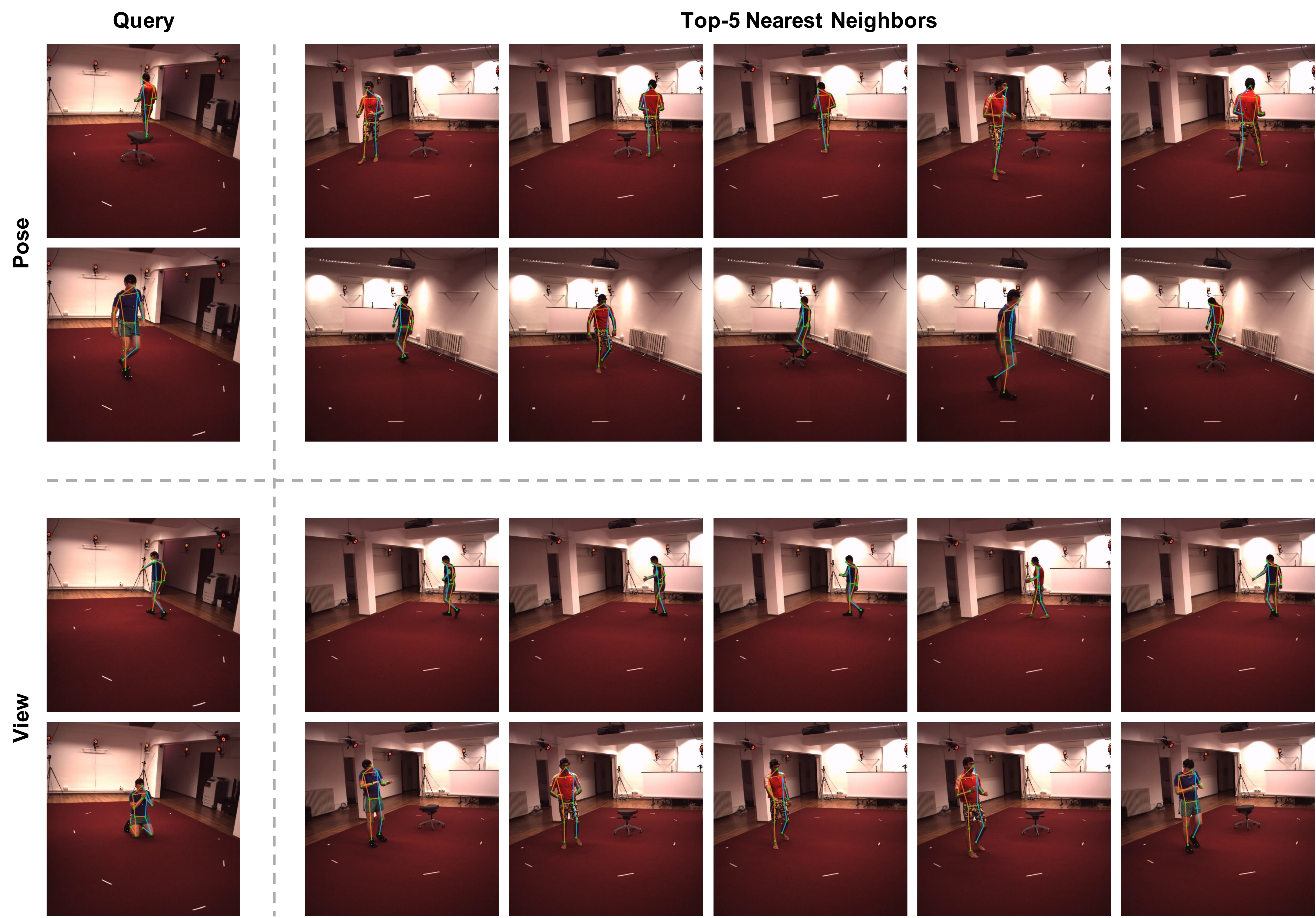}
\end{center}
  \caption{Nearest neighbors in the representation space using subjects S9 and S11 on Human3.6M~\cite{ionescu2013human3}. The first two rows use pose representations, while the second two rows use view representations. On each row, we show the query on the left and its top five nearest neighbors on the right.}
\label{fig:supp:h36m}
\end{figure*}

\begin{figure*}[t]
\begin{center}
  \includegraphics[width=\linewidth]{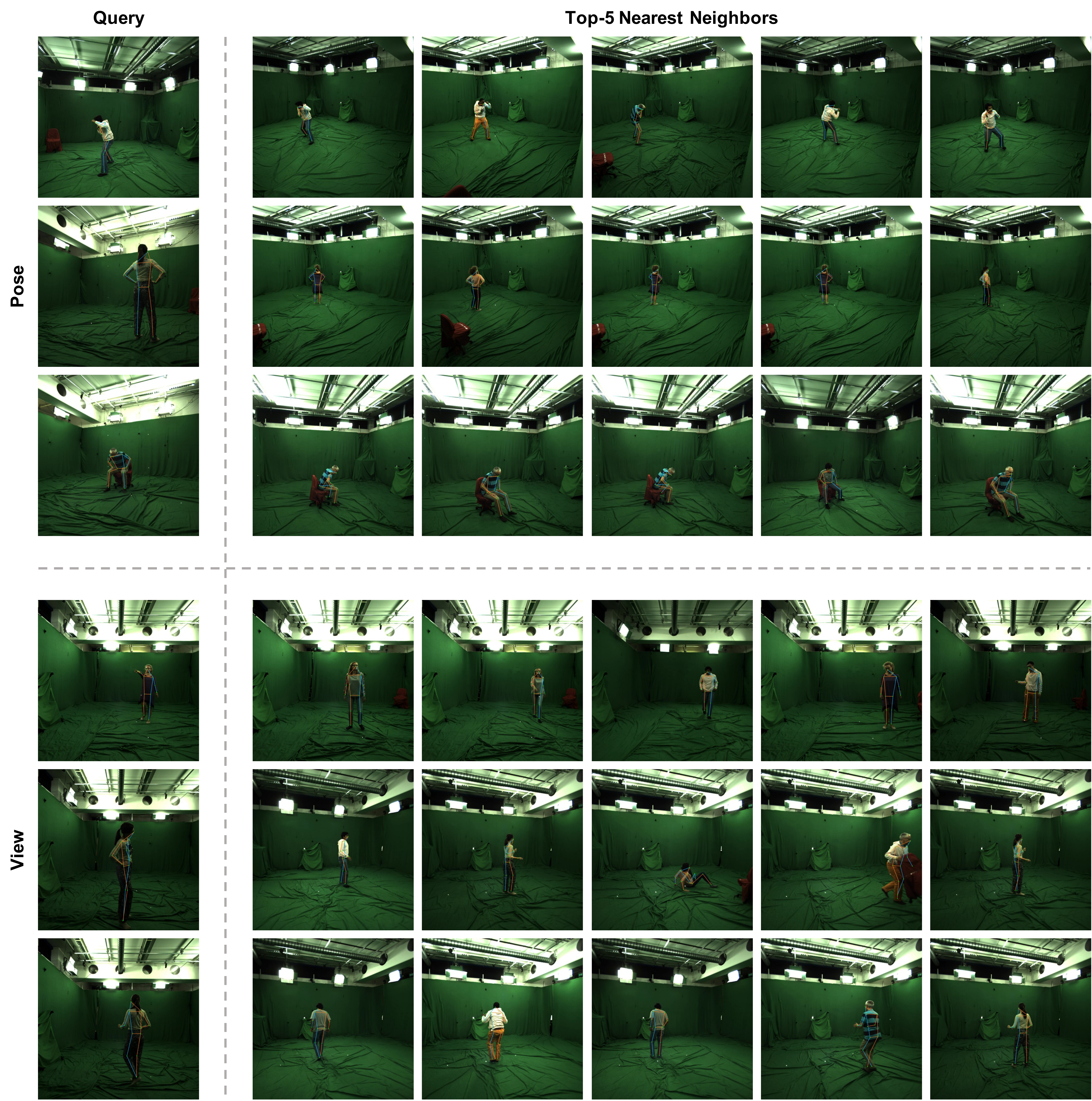}
\end{center}
  \caption{Nearest neighbors in the representation space on MPI-INF-3DHP~\cite{mehta2017monocular}. The first three rows use pose representations, while the second three rows use view representations. On each row, we show the query on the left and its top five nearest neighbors on the right.}
\label{fig:supp:3dhp}
\end{figure*}